\documentclass[envcountsame, oribibl]{llncs}

\usepackage{cite}
\usepackage[UKenglish]{babel}

\usepackage{amsmath,mathtools,amssymb}
\usepackage{nicefrac}
\usepackage{enumitem}

\title{An Imprecise Probabilistic Estimator for the Transition Rate Matrix of a Continuous-Time Markov Chain}
\titlerunning{An IP Estimator for the Rate Matrix of a CTMC}
\author{Thomas Krak \and Alexander Erreygers \and Jasper De Bock}
\institute{Ghent University, ELIS, SYSTeMS \email{\string{thomas.krak,alexander.erreygers,jasper.debock\string}@ugent.be}}

\newcommand{\nats}{\bbbn}
\newcommand{\natswith}{\nats_{0}}
\newcommand{\reals}{\bbbr}

\newcommand{\realspos}{\reals_{>0}}
\newcommand{\realsnonneg}{\reals_{\geq 0}}

\newcommand{\states}{\mathcal{X}}

\newcommand{\ind}[1]{{\mathbb{I}}_{#1}}

\newcommand{\prev}{\mathbb{E}}

\newcommand{\timedim}{\bbbt}

\newcommand{\partobs}{\widetilde{\omega}}

\newcommand{\discretestep}{\smash{\delta^{(m)}}}
\newcommand{\partobsdisc}{\smash{w^{(m)}}}
\newcommand{\discretecounts}[1]{\smash{n_{#1}^{(m)}}}
\newcommand{\discretetranscount}[1]{\smash{n_{#1}^{(m)}}}
\newcommand{\discretetransmat}{\smash{T^{(m)}}}
\newcommand{\discretetransmatML}{\smash{T^{(m),\mathrm{ML}}}}

\newcommand{\transmatset}{\mathfrak{T}}
\newcommand{\ratematset}{\mathfrak{Q}}

\DeclarePairedDelimiterXPP{\varnorm}[1]{}{\lVert}{\rVert}{_{v}}{#1}

\DeclarePairedDelimiterX\gr[1](){ #1}

\begin{document}

\maketitle

\begin{abstract}
We consider the problem of estimating the transition rate matrix of a continuous-time Markov chain from a finite-duration realisation of this process. We approach this problem in an imprecise probabilistic framework, using a set of prior distributions on the unknown transition rate matrix. The resulting estimator is a set of transition rate matrices that, for reasons of conjugacy, is easy to find. To determine the hyperparameters for our set of priors, we reconsider the problem in discrete time, where we can use the well-known Imprecise Dirichlet Model. In particular, we show how the limit of the resulting discrete-time estimators is a continuous-time estimator. It corresponds to a specific choice of hyperparameters and has an exceptionally simple closed-form expression.
\end{abstract}

\section{Introduction}

Continuous-time Markov chains (CTMCs) are mathematical models that describe the evolution of dynamical systems under (stochastic) uncertainty~\cite{norris1998markov}. They are pervasive throughout science and engineering, finding applications in areas as disparate as medicine, mathematical finance, epidemiology, queueing theory, and others. We here consider time-homogeneous CTMCs that can only be in a finite number of states.

The dynamics of these models are uniquely characterised by a single \emph{transition rate matrix} $Q$. This $Q$ describes the (locally) linearised dynamics of the model, and is the generator of the semi-group of transition matrices $T_t=\exp(Qt)$ that determines the conditional probabilities $P(X_t=y\,\vert\,X_0=x)=T_t(x,y)$. In this expression, $X_t$ denotes the uncertain state of the system at time $t$, and so $T_t$ contains the probabilities for the system to move from any state $x$ at time zero to any state $y$ at time $t$.

In this work, we consider the problem of estimating the matrix $Q$ from a single realisation of the system up to some finite point in time. This problem is easily solved in both the classical frequentist and Bayesian frameworks, due to the likelihood of the corresponding CTMC belonging to an exponential family; see e.g. the introductions of~\cite{inamura2006estimating, bladt2005statistical}. The novelty of the present paper is that we instead consider the estimation of $Q$ in an \emph{imprecise probabilistic}~\cite{Walley:1991vk,augustin2013:itip} context.

Specifically, we approach this problem by considering an entire \emph{set} of Bayesian priors on the likelihood of $Q$, leading to a \emph{set-valued} estimator for $Q$.
In order to obtain well-founded hyperparameter settings for this set of priors, we recast the problem by interpreting a continuous-time Markov chain as a limit of \emph{discrete}-time Markov chains. 
This allows us to consider the imprecise-probabilistic estimators of these discrete-time Markov chains, which are described by the popular Imprecise Dirichlet Model~(IDM)~\cite{Quaeghebeur:2009wm}. The upshot of this approach is that the IDM has well-known prior hyperparameter settings which can be motivated from first principles~\cite{Walley:1996vt,deCooman:2015vj}. 

This leads us to the two main results of this work. First of all, we show that the limit of these IDM estimators is a set of transition rate matrices that can be described in closed-form using a very simple formula. Secondly, we identify the hyperparameters of our imprecise CTMC prior such that the resulting estimator is equivalent to the estimator obtained from this discrete-time limit. The proofs of our results can be found in the appendix.

The immediate usefulness of our results is two-fold. From a domain-analysis point of view, where we are interested in the parameter values of the process dynamics, our imprecise estimator provides prior-insensitive information about these values based on the data. If we are instead interested in robust inference about the future behaviour of the system, our imprecise estimator can be used as the main parameter of an \emph{imprecise continuous-time Markov chain}~\cite{Skulj:2015cq,DeBock:2016,krak2016ictmc,erreygers2017imprecise}. 

\section{A Brief Refresher on Stochastic Processes}\label{subsec:stoch_proc}

Intuitively, a stochastic process describes the uncertainty in a stochastic system's behaviour as it moves through some state space $\states$ as time $t$ progresses over some time dimension $\timedim$. A fundamental choice is whether we are considering processes in discrete time, in which case typically $\timedim=\natswith$, or in continuous time, in which case $\timedim=\realsnonneg$. Here we write $\nats$ for the natural numbers, and let $\natswith\coloneqq \nats\cup\{0\}$. The real numbers are denoted by $\reals$, the positive reals by $\realspos$, and the non-negative reals by $\realsnonneg$. We briefly recall the basic definitions of stochastic processes below; for an introductory work we refer to e.g.~\cite{norris1998markov}. 

Formally, a realisation of a stochastic process is a \emph{sample path}, which is a map $\omega\,:\,\timedim\to\states$. 
Here $\omega(t)\in\states$ represents the state of the process at time $t\in\timedim$. We collect all sample paths in the set $\Omega$ and, when $\timedim=\realsnonneg$, these paths are assumed to be c\`adl\`ag under the discrete topology on $\states$.  
With this domain in place, we then consider some abstract underlying probability space $(\Omega,\mathcal{F},P)$, where $\mathcal{F}$ is some appropriate ($\sigma$-)algebra on $\Omega$, and where $P$ is a (countably-)additive probability measure.

The stochastic process can now finally be defined as a family of random variables $\{X_t\}_{t\in\timedim}$ associated with this probability space. In particular, for fixed $t\in\timedim$, the quantity $X_t$ is a random variable $\Omega\to\states\,:\,\omega\mapsto\omega(t)$. Conversely, for a fixed realisation $\omega\in\Omega$, $X_t(\omega)$ is a deterministic map $\timedim\to\states\,:\,t\mapsto\omega(t)$.

Well-known and popular kinds of stochastic processes are \emph{Markov chains}:
\begin{definition}[Markov Chain]
Fix $\timedim\in\{\natswith,\realsnonneg\}$, and let $\{X_t\}_{t\in\timedim}$ be a stochastic process. We call this process a \emph{Markov chain} if, for all $s_0,\ldots,s_n,s,t\in\timedim$ for which $s_0< \cdots< s_n < s< t$, it holds that
$P(X_t=x_t\,\vert\,X_{s_0}=x_{s_0}, \ldots, X_{s_n}=x_{s_n},X_s=x_s)=P(X_t=x_t\,\vert\,X_s=x_s)$ for all $x_{s_0},\ldots,x_{s_n},x_s,x_t\in\states$.
If then $\timedim=\natswith$, we call $\{X_t\}_{t\in\timedim}$ a \emph{discrete-time Markov chain} (DTMC). If instead $\timedim=\realsnonneg$, we call it a \emph{continuous-time Markov chain} (CTMC).
\end{definition}
Furthermore, attention is often restricted to \emph{homogenous} Markov chains:
\begin{definition}[Homogeneous Markov Chain] Let $\{X_t\}_{t\in\timedim}$ be a Markov chain. We call this Markov chain \emph{(time-)homogeneous} if, for all $s,t\in\timedim$, $s\leq t$,  and all $x,y\in\states$, it holds that $P(X_t=y\,\vert\,X_s=x)=P(X_{(t-s)}=y\,\vert\,X_0=x)$.
\end{definition}
This homogeneity property makes such processes particularly easy to describe. 

In what follows, we will say that a $\lvert\states\rvert\times\lvert\states\rvert$ matrix $T$ is a \emph{transition matrix}, if it is a real-valued and row stochastic matrix, i.e. if $T(x,y)\geq 0$ and $\sum_{z\in\states}T(x,z)=1$ for all $x,y\in\states$. We write $\transmatset$ for the space of all transition matrices. The elements $T$ of $\transmatset$ can be used to describe the single-step conditional probabilities of a (homogeneous) DTMC:
\begin{proposition}[\hspace{-0.5pt}\cite{norris1998markov}]\label{prop:trans_gives_homogen_dtmc}
Let $\{X_t\}_{t\in\natswith}$ be a homogeneous DTMC. Then this process is completely and uniquely characterised by a probability mass function $p$ on $\states$ and some $T\in\transmatset$. In particular, $P(X_0)=p$ and, for all $t\in\natswith$ and all $x,y\in\states$, $P(X_{t}=y\,\vert\,X_0=x)=T^t(x,y)$, where $T^t$ is the $t^{\text{th}}$ matrix power of $T$.
\end{proposition}

On the other hand, to describe CTMCs we need the concept of a \emph{(transition) rate matrix}: a $\lvert\states\rvert\times\lvert\states\rvert$ real-valued matrix $Q$ with non-negative off-diagonal elements and zero row-sums, i.e. $Q(x,y)\geq 0$ and $\sum_{z\in\states}Q(x,z)=0$ for all $x,y\in\states$ such that $x\neq y$. We write $\ratematset$ for their entire space. A rate matrix describes the ``speed'' with which a CTMC moves between its states:
\begin{proposition}[\hspace{-0.5pt}\cite{norris1998markov}]\label{prop:rate_gives_homogen_ctmc}
Let $\{X_t\}_{t\in\realsnonneg}$ be a homogeneous CTMC. Then this process is completely and uniquely characterised by a probability mass function $p$ on $\states$ and some $Q\in\ratematset$. In particular, $P(X_0)=p$ and, for all $t\in\realsnonneg$ and all $x,y\in\states$, $P(X_t=y\,\vert\,X_0=x)=\exp(Qt)(x,y)$, where $\exp(Qt)$ is the matrix exponential of $Qt$. Furthermore, for small enough $\Delta\in\realsnonneg$ and all $x,y\in\states$, it holds that $P(X_\Delta=y\,\vert\,X_0=x)\approx (I+\Delta Q)(x,y)$, where $I$ is the identity matrix.
\end{proposition}
\quad\vspace{-30pt}

\section{Estimation of a CTMC's Rate Matrix}\label{sec:estimation_ct}

In what follows, we will derive methods to estimate the rate matrix $Q$ of a homogeneous CTMC from a realisation $\omega\in\Omega$ that was observed up to some finite point in time $t_\mathrm{max}\in\realsnonneg$. We denote with $\partobs$ the restriction of $\omega$ to this interval $[0,t_\mathrm{max}]\subset\realsnonneg$, and we consider this (finite-duration) observation to be fixed throughout the remainder of this paper. 

For any $x,y\in\states$ such that $x\neq y$, we let $n_{xy}$ denote the number of transitions from state $x$ to state $y$ in $\partobs$. 
Furthermore, we let $d_x$ denote the total duration spent in state $x$, that is, we let $d_x\coloneqq \int_{0}^{t_\mathrm{max}}\ind{x}(\partobs(t))\,\mathrm{d}t$, where $\ind{x}$ is the indicator of $\{x\}$, defined by $\ind{x}(\partobs(t))\coloneqq 1$ if $\partobs(t)=x$ and $\ind{x}(\partobs(t))\coloneqq 0$ otherwise. We assume in the remainder that $d_x>0$ for all $x\in\states$. Finally, for notational convenience, we define $q_{xy}\coloneqq Q(x,y)$ for all $x,y\in\states$.

\vspace{-2pt}
\subsection{Precise Estimators}\label{subsec:freq_ctmc}

Under the assumption that the realisation $\omega$ was generated by a homogeneous continuous-time Markov chain with rate matrix $Q$, it is well known that the process dynamics can be modelled using exponentially distributed random variables whose parameters are given by the elements of Q. For various of such interpretations, we refer to e.g.~\cite{norris1998markov}. What matters to us here is that, regardless of the interpretation, we can use this to obtain the following likelihood result (see e.g.~\cite{inamura2006estimating}): for a given $\partobs$, the likelihood for a rate matrix $Q$ is 
\vspace{-2pt}
\begin{equation}\label{eq:ctmc_likelihood}
L(\partobs\,\vert\, Q) = \prod_{\substack{x,y\in\states\\x\neq y}}(q_{xy})^{n_{xy}}e^{-q_{xy}d_x}.
\vspace{-4pt}
\end{equation}
The corresponding maximum-likelihood estimator $\smash{Q^\mathrm{ML}}$ is easily found~\cite{inamura2006estimating}: $\smash{q_{xy}^{\mathrm{ML}}}=\nicefrac{n_{xy}}{d_x}$ if $x\neq y$ and $\smash{q_{xx}^\mathrm{ML}}=-\smash{\sum_{y\in\states\setminus\{x\}}q_{xy}^{\mathrm{ML}}}$, where the final expression follows from the (implicit) constraint that the rows of a rate matrix should sum to zero.

Inspection of the likelihood in~\eqref{eq:ctmc_likelihood} reveals that it belongs to an exponential family. This implies that there exists a conjugate prior for the rate matrix $Q$, such that its posterior distribution, given $\partobs$, belongs to the same family as this prior. This prior is given by a product of Gamma distributions, specifically on the off-diagonal elements $q_{xy}$, $x\neq y$, of the corresponding rate matrix~\cite{bladt2005statistical}. We here use a slightly more general joint prior on $Q$ whose ``density'' $f$ is given by
\begin{equation}\label{eq:gamma_prior}
f(Q\,\vert\,\boldsymbol{\alpha},\boldsymbol{\beta}) \coloneqq \prod_{\substack{x,y\in\states\\x\neq y}} (q_{xy})^{\alpha_{xy}-1}e^{-q_{xy}\beta_{x}} \propto \prod_{\substack{x,y\in\states\\x\neq y}} \mathrm{Gamma}(q_{xy}\,\vert\,\alpha_{xy},\beta_{x}),
\vspace{-3pt}
\end{equation}
with shapes $\alpha_{xy}$ and rates $\beta_{x}$ in $\realsnonneg$; we write $\boldsymbol{\alpha},\boldsymbol{\beta}$ for the joint parameters.

Note that we have only defined the prior to equal a product of Gamma distributions up to normalisation, so that the prior $f(Q\,\vert\,\boldsymbol{\alpha},\boldsymbol{\beta})$ may be improper. This has the advantage that it allows us to close the parameter domains and allow prior hyperparameters $\alpha_{xy}=0$ and $\beta_{x}=0$, for which the Gamma distribution is not properly defined. We acknowledge that the use of such improper priors is not entirely uncontroversial, and that their interpretation as a prior probability (which it indeed is not) leaves something to be desired. We will nevertheless, in this specific setting, be able to motivate their use here as a consequence of Theorem~\ref{thm:limit_of_idm} further on.

Also, despite being improper, we can of course combine the prior~\eqref{eq:gamma_prior} with the likelihood~\eqref{eq:ctmc_likelihood} and fix the normalisation in the posterior. As is well known, the means of the marginals of this posterior are then of the form\footnote{The assumption $d_x>0$ prevents division by zero in~\eqref{eq:posterior_gamma_mean}. However, $n_{xy}$ might be zero and, if then also $\alpha_{xy}=0$, the posterior cannot be normalised and will still be improper. Nevertheless, using an intuitive (but formally cumbersome) argument we can still identify this posterior for $q_{xy}$ with the (discrete) distribution putting all mass at zero. Alternatively, we can motivate~\eqref{eq:posterior_gamma_mean} by continuous extension from the cases where $\alpha_{xy}>0$, similarly yielding the estimate $\hat{q}_{xy}=0$ at $\alpha_{xy}=n_{xy}=0$. }
\begin{equation}\label{eq:posterior_gamma_mean}
\prev\bigl[q_{xy}\,\vert\,\boldsymbol{\alpha},\boldsymbol{\beta},\partobs\bigr] = \frac{\alpha_{xy}+n_{xy}}{\beta_{x}+d_x}\quad\quad \forall{x,y\in\states,x\neq y}.
\end{equation}
Furthermore, the (joint) posterior mean is well-known to be a Bayes estimator for $Q$ under quadratic loss and given the prior $f(\,\cdot\,\vert\,\boldsymbol{\alpha},\boldsymbol{\beta})$~\cite{berger1985decisiontheory}.

The question now remains of how to a priori settle on a ``good'' choice for these hyperparameters $\boldsymbol{\alpha},\boldsymbol{\beta}$, in the sense that they should adequately represent our prior beliefs. This is a non-trivial problem, and no general solution can be given. 
A popular (but not uncontroversial) attempt to characterise a non-informative prior consists in choosing the improper prior with $\boldsymbol{\alpha}=\boldsymbol{\beta}=0$; the posterior mean (Bayes) estimator then equals $Q^\mathrm{ML}$.

\vspace{-10pt}
\subsection{An Imprecise Probabilistic Estimator}

Generalising the above Bayesian approach, we here suggest an imprecise probabilistic treatment. Following for example\cite{Walley:1991vk,quaeghebeur2005exponential}, this approach consists in using an entire \emph{set} of prior distributions. Specifically, we consider a set of the form
\begin{equation}\label{eq:set_of_gamma}
\left\{ f(\,\cdot\,\vert\,\boldsymbol{\alpha},\boldsymbol{\beta})\,\Big\vert\,
(\boldsymbol{\alpha}, \boldsymbol{\beta})\in C \right\},
\end{equation}
with $f(\,\cdot\,\vert\,\boldsymbol{\alpha},\boldsymbol{\beta})$ as in~\eqref{eq:gamma_prior}, and where $C$ is a set of possible prior parameters. In this way, we do not have to restrict our attention to one specific choice of the parameters $\boldsymbol{\alpha},\boldsymbol{\beta}$; rather, we can include all the parameter settings that we deem reasonable, by collecting them in $C$. Inference from $\partobs$ is then performed by point-wise updating each of these priors; we thereby obtain a set of posterior distributions on the space of all rate matrices. Each of these posteriors has a mean of the form~\eqref{eq:posterior_gamma_mean}, which is a Bayes estimator for $Q$ under a specific prior in the set~\eqref{eq:set_of_gamma}. This leads us to consider the imprecise, i.e., \emph{set-valued}, estimator
\begin{align*}
\mathcal{Q}_C\coloneqq \Biggl\{ &Q\in\ratematset \Bigg\vert \left( \forall{x,y\in\states,x\neq y} :q_{xy}=\frac{\alpha_{xy}+n_{xy}}{\beta_{x}+d_x} \right), (\boldsymbol{\alpha},\boldsymbol{\beta})\in C \Biggr\}\,.
\end{align*}

Note that even in this imprecise probabilistic approach, we still need to somehow specify the (now set-valued) prior model. That is, we need to be specific about the set $C$. Inspired by the well-known imprecise Dirichlet model~\cite{Walley:1996vt}, we may choose an ``imprecision parameter'' $s\in\realsnonneg$, which can be interpreted as a number of ``pseudo-counts'', to constrain $0\leq \sum_{y\in\states\setminus\{x\}}\alpha_{xy}\leq s$ for all $x\in\states$, and to then vary all $\beta_{x}$ over their domain $\realsnonneg$.    
Unfortunately, similar to what is noted in~\cite{quaeghebeur2005exponential}, this leads to undesirable behaviour. For example, as is readily seen from e.g.~\eqref{eq:posterior_gamma_mean}, including unbounded $\beta_{x}$ allows the off-diagonal elements $q_{xy}$ to get arbitrarily close to zero, causing the model to a posteriori believe that transitions leaving $x$ may be impossible, no matter the number of such transitions that we actually observed in $\partobs$! Hence, we prefer a different choice of $C$. 

One way to circumvent this undesired behaviour is to constrain the range within which each $\beta_x$ may be varied, to some interval $[0,\overline{\beta}_x]$, say. The downside is that this introduces a large number of additional hyperparameters; we then need to (``reasonably'') choose a value $\overline{\beta}_x\in\realsnonneg$ for each $x\in\states$. Fortunately, our main result -- Theorem~\ref{thm:limit_of_idm} further on -- suggests that setting $\overline{\beta}_x=0$ (and therefore $\beta_x=0$) is in fact a very reasonable choice. This identification is obtained in the next section, using a limit result of discrete-time estimators, for which the hyperparameter settings follow entirely from first principles.

In summary, we keep the ``imprecision parameter'' $s\in\realsnonneg$ and the constraint $0\leq \sum_{y\in\states\setminus\{x\}}\alpha_{xy}\leq s$ for all $x\in\states$, and simply set $\beta_x=0$ for all $x\in\states$. We then define $C_s$ to be the largest set of parameters that satisfies these properties.
Every $\boldsymbol{\alpha}$ in this set can be conveniently identified with the off-diagonal elements of a matrix $sA$, with $A\in\transmatset$ a transition matrix. Our set-valued estimator $\mathcal{Q}_s$ can thus be written as
\vspace{-4pt}
\begin{equation}\label{eq:proposed_estimator}
\mathcal{Q}_s \coloneqq \left\{ Q\in\ratematset \,\Bigg\vert\, \left(\forall x,y\in\states,x\neq y\,:\,q_{xy}=\frac{sA(x,y)+n_{xy}}{d_x}\right),\, A\in\transmatset\right\}.
\vspace{-10pt}
\end{equation}

\section{Discrete-Time Estimators and Limit Relations}

A useful intuition is that we can consider a CTMC as a limit of DTMCs, where we assign increasingly shorter durations to the time steps at which the latter operate. In this section, we will use this connection to relate estimators for DTMCs to estimators for CTMCs. We start by discretising the observed path.

Because the realisation $\omega$ was only observed up to some time $t_\mathrm{max}\in\realsnonneg$, we can discretise the (finite-duration) realisation $\partobs$ into a finite number of steps. For any $m\in\nats$, we write $\discretestep\coloneqq \nicefrac{t_\mathrm{max}}{m}$, and we define the discretised path $\partobsdisc\,:\,\{0,\ldots,m\}\to\states$ as $\partobsdisc(i)\coloneqq \partobs\left(i\discretestep\right)$ for all $i\in\{0,\ldots, m\}$.

For any $m\in\nats$ and $x,y\in\states$, we let $\discretetranscount{xy}\coloneqq\sum_{i=1}^m\ind{x}(\partobsdisc(i-1))\ind{y}(\partobsdisc(i))$ denote the number of transitions from state $x$ to $y$ in $\partobsdisc$, and we let $\discretecounts{x}\coloneqq \sum_{y\in\states}\discretetranscount{xy}$ denote the total number of time steps that started in state $x$. 

\vspace{-10pt}
\subsection{Discrete-Time Estimators}

For fixed $m\in\nats$, we can interpret the discretised path $\partobsdisc$ as a finite-duration ($m+1$ steps long) realisation of a homogeneous discrete-time Markov chain with transition matrix $\discretetransmat$, with $m$ keeping track of the discretisation level. Each transition along the path $\partobsdisc$, from state $x$ to $y$, say, is then a realisation of a categorical distribution with parameters $\discretetransmat(x,\cdot)$. The likelihood for $\discretetransmat$, given $\partobsdisc$, is therefore proportional to a product of independent multinomial likelihoods. Hence, the maximum likelihood estimator follows straightforwardly and as expected: $\discretetransmatML(x,y)=\nicefrac{\discretetranscount{xy}}{\discretecounts{x}}$ for all $x,y\in\states$; see~\cite{inamura2006estimating} for details.

In a Bayesian analysis, and following e.g.~\cite{Quaeghebeur:2009wm}, for fixed $m$ we can model our uncertainty about the unknown $\discretetransmat$ by putting independent Dirichlet priors on the rows $\discretetransmat(x,\cdot)$. We write this prior as $g(\cdot\,\vert\,s,A)$, where $s\in\realsnonneg$ is a ``prior strength'' parameter, and $A\in\mathrm{int}(\transmatset)$ is a prior location parameter. 
Note that we take $A$ in the interior of $\transmatset$ -- under the metric topology on $\transmatset$ -- so that each row $A(x,\cdot)$ corresponds to a strictly positive probability mass function.

After updating with $\partobsdisc$, the posterior mean is an estimator for $\discretetransmat$ that is Bayes under quadratic loss and for the specific prior $g(\,\cdot\,\vert\,s,A)$; due to conjugacy, the posterior is again a product of independent Dirichlet distributions~\cite{Quaeghebeur:2009wm}, whence the elements of the posterior mean are
\begin{equation*}
\mathbb{E}\Bigl[T(x,y)\,\Big\vert\,s,A,\partobsdisc\Bigr] = \frac{sA(x,y)+\discretetranscount{xy}}{s+n_x^{(m)}}\quad\quad\forall x,y\in\states\,.
\end{equation*}
\vspace{-10pt}

What remains is again to determine a good choice for $s$ and $A$. However, in an imprecise probabilistic context we do not have to commit to any such choice: the popular Imprecise Dirichlet Model generalises the above approach using a \emph{set} of Dirichlet priors.  
This set is given by $\smash{\mathrm{IDM}(\,\cdot\,\vert\,s) \coloneqq \bigl\{ g(\,\cdot\,\vert\,s,A)\,\big\vert\, A\in\mathrm{int}(\transmatset) \bigr\}}$ and can be motivated from first principles~\cite{Walley:1996vt,deCooman:2015vj}.
Observe that only a parameter $s\in\realsnonneg$ remains, which controls the ``degree of imprecision''. In particular, we no longer have to commit to a location parameter $A$; instead this parameter is freely varied over its entire domain $\mathrm{int}(\transmatset)$.

Element-wise updating with $\partobsdisc$ yields a set of posteriors which, due to conjugacy, are again independent products of Dirichlet distributions. 
The corresponding set $\smash{\mathcal{T}^{(m)}_s}$ of posterior means thus contains estimators for $\discretetransmat$ that are Bayes for a specific prior from the IDM, and is easily verified to be
\begin{equation*}
\mathcal{T}^{(m)}_s = \left\{ T\in\transmatset \,\Bigg\vert\, \left(\forall{x,y\in\states}:T(x,y)=\frac{sA(x,y)+n_{xy}^{(m)}}{s+n_{x}^{(m)}} \right), A\in\mathrm{int}(\transmatset)\right\}\,.
\end{equation*}
\vspace{-20pt}
\subsection{Limits of Discrete-Time Estimators}

As noted in Proposition~\ref{prop:rate_gives_homogen_ctmc}, a rate matrix $Q$ is connected to the transition probabilities $T_\Delta(x,y)\coloneqq P(X_\Delta=y\,\vert\,X_0=x)$ in the sense that $T_\Delta\approx (I+\Delta Q)$ for small $\Delta$. Hence, for small $\Delta$, we have that $Q\approx(T_\Delta - I)\nicefrac{1}{\Delta}$. This becomes exact in the limit for $\Delta$ going to zero.

This interpretation can also be used to connect discrete-time estimators for $\discretetransmat$ to estimators for $Q$. 
For example, if we let $\smash{Q^{(m)}}\coloneqq (\discretetransmatML-I)\nicefrac{1}{\discretestep}$, then $Q^\mathrm{ML}=\lim_{m\to+\infty}Q^{(m)}$.
Similarly, we can connect our set-valued estimators for the discretised models to the set-valued continuous-time estimator in~\eqref{eq:proposed_estimator}: 
\begin{theorem}\label{thm:limit_of_idm}
For all $m\in\nats$, let $\smash{\mathcal{Q}^{(m)}_s}\coloneqq \smash{\bigl\{(T-I)\nicefrac{1}{\discretestep}\,\big\vert\,T\in\mathcal{T}^{(m)}_s\bigr\}}$. Then the Painlev\'e-Kuratowski\cite{rockafellar1998variational} limit $\lim_{m\to+\infty}\mathcal{Q}^{(m)}_s$ exists, and equals $\mathcal{Q}_s$.
\end{theorem}
\quad\vspace{-30pt}

\section{Discussion}

We have derived a set-valued estimator $\mathcal{Q}_s$ for the transition rate matrix of a homogeneous CTMC.  
It can be motivated both as a set of posterior means of a set of Bayesian models in continuous-time, and as a limit of set-valued discrete-time estimators based on the Imprecise Dirichlet Model. 
The only parameter of the estimator is a scalar $s\in\realsnonneg$ that controls the degree of imprecision. In the special case where $s=0$ there is no imprecision, and then $\smash{\mathcal{Q}_0}=\{\smash{Q^\mathrm{ML}}\}$.

The set-valued representation $\mathcal{Q}_s$ is convenient when one is interested in the numerical values of the transition rates, e.g. for domain-analysis. If one aims to use the estimator to describe an \emph{imprecise CTMC}~\cite{Skulj:2015cq,krak2016ictmc}, a representation using the \emph{lower transition rate operator} $\underline{Q}$ is more convenient.  
This operator is the lower envelope of a set of rate matrices; for $\smash{\mathcal{Q}_s}$ it is given, for all $h:\states\to\reals$, by
\begin{equation*}
\bigl[\,\underline{Q}h\bigr](x) \coloneqq \hspace{-5pt}\inf_{Q\in \mathcal{Q}_s} \sum_{y\in\states}\hspace{-3pt}Q(x,y)h(y) = \frac{s}{d_x}\min_{y\in\states}\bigl(h(y)-h(x)\bigr) + \hspace{-5pt}\sum_{y\in\states\setminus\{x\}}\hspace{-5pt}\frac{n_{xy}}{d_x}\bigl(h(y)-h(x)\bigr),
\vspace{-6pt}
\end{equation*}
for all $x\in\states$. Hence, $\underline{Q}h$ is straightforward to evaluate. This implies that when our estimator is used to learn an imprecise CTMC from data, the lower expectations of this imprecise CTMC can be computed efficiently~\cite{erreygers2017imprecise}. 
\vspace{-12pt}

	\section*{Acknowledgements}
The work in this paper was partially supported by H2020-MSCA-ITN-2016 UTOPIAE, grant agreement 722734. The authors wish to thank two anonymous reviewers for their helpful comments and suggestions.

\newpage
\appendix

\section{Proofs of Main Results}

In this appendix, we will assume that the paths $\omega\in\Omega$ are functions in continuous-time, that is, that $\timedim=\realsnonneg$. As stated in Section~\ref{subsec:stoch_proc}, we then assume all these paths to be c\`adl\`ag under the discrete topology on $\states$.
 
Therefore, and because the realisation $\partobs$ is only observed up to time $t_\mathrm{max}$, there are only a finite number of state transitions in $\partobs$ and, furthermore, each distinct visit lasts for a strictly positive (but finite) duration; the lemma below makes this formal. The result is essentially well-known, but we had some trouble finding a satisfactory reference for the elementary case where $\partobs$ takes at most finitely many values; we therefore prove it as a (somewhat trivial) specialisation of~\cite[Theorem 12.2.1]{whitt2002stochastic}.
\begin{lemma}\label{lemma:cadlag_has_finite_switches}
Let $\partobs$ be c\`adl\`ag. Then there is a finite collection of time points $t_i\in [0,t_\mathrm{max}]$, $i=0,\ldots,M$, with $t_0=0$, $t_M=t_\mathrm{max}$, and $t_i<t_j$ if $i<j$, such that $\partobs(t)$ is constant on $[t_{i},t_{i+1})$ for all $i=0,\ldots,M-1$, and $\partobs(t_{i-1})\neq \partobs(t_i)$ for all $i=1,\ldots,M-1$.
\end{lemma}
\begin{proof}
By assumption $\partobs$ is c\`adl\`ag on $[0,t_\mathrm{max}]$ under the discrete topology on $\states$. Since $\states$ is finite, we can identify it without loss of generality with the set $\{1,\ldots,k\}\subset\nats$, with $k=\lvert\states\rvert$ the number of states.
Now let $r:[0,t_\mathrm{max}]\to\reals$ be defined as $r(t)\coloneqq\partobs(t)$ for all $t\in[0,t_\mathrm{max}]$, so that $r$ takes values in $\{1,\ldots,k\}\subset\nats\subset\reals$. Then $r$ is simply $\partobs$ with its co-domain $\states$ replaced by $\reals$; $r$ is therefore by construction c\`adl\`ag under the discrete topology on $\reals$. 

For any $x\in\reals$, any open neighbourhood $U_x\subseteq\reals$ of $x$ in the Euclidean topology contains the set $\{x\}$, which is an open neighbourhood of $x$ in the discrete topology. It follows that any sequence $\{x_i\}_{i\in\nats}$ in $\reals$ that is convergent with limit $x_*$ in the discrete topology, is also convergent with limit $x_*$ in the Euclidean topology; the sequence $\{x_i\}_{i\in\nats}$ is eventually in $\{x_*\}\subset U_{x_*}$ for any open neighbourhood $U_{x_*}$ of $x_*$ in the Euclidean topology. Therefore the left-sided limits and right-continuity of $r$ under the discrete topology, hold identically under the Euclidean topology; so $r$ is also c\`adl\`ag under the Euclidean topology on $\reals$.

By~\cite[Theorem 12.2.1]{whitt2002stochastic}, $r$ has at most a finite number of discontinuities, under the Euclidean norm on $\reals$, on the interval $[0,t_\mathrm{max}]$. Denote these points of discontinuity as $t_1,\ldots,t_{M'}$, and assume without loss of generality that $t_i<t_j$ for all $i,j\in\{1,\ldots,M'\}$ for which $i<j$. We next include the endpoints of the interval. Note first that $t_1>0$, because time $0$ cannot be a point of discontinuity due to the c\`adl\`ag property; we can therefore introduce $t_0\coloneqq 0$. For the endpoint $t_\mathrm{max}$ we need to consider two cases, because there is possibly a discontinuity there. If $t_{M'}\neq t_\mathrm{max}$ there is no such discontinuity, whence we introduce $t_M\coloneqq t_\mathrm{max}$ and set $M\coloneqq M'+1$; otherwise we simply let $M\coloneqq M'$. We will now verify the properties in the lemma's statement.

Clearly, by construction, we have that $t_i\in[0,t_\mathrm{max}]$ for all $i=0,\ldots,M$, that $t_0=0$ and $t_M=t_\mathrm{max}$ and that $t_i<t_j$ if $i<j$. Now fix any $i\in\{0,\ldots,M-1\}$; we know that $r$ has no discontinuities on the interval $(t_i,t_{i+1})$ under the Euclidean norm on $\reals$, and therefore, since $r$ is right-continuous at $t_i$, it has no discontinuities on $[t_i,t_{i+1})$ either. Since $r$ only takes values in $\{1,\ldots,k\}$, it follows that $r$ must be constant on $[t_i,t_{i+1})$. This implies that also $\partobs$ is constant on $[t_i,t_{i+1})$. 

Finally, choose any $i\in\{1,\ldots,M-1\}$. There is then a discontinuity in $r$ at $t_{i}$. Since, as we have just shown, $r$ is constant on both $[t_{i-1},t_i)$ and $[t_{i},t_{i+1})$, this implies that $r(t_{i-1})\neq r(t_i)$, and therefore also $\partobs(t_{i-1})\neq \partobs(t_i)$.\qed
\end{proof}

In other words, for $i\in\{1,\ldots,M-1\}$, the time points $t_i$ are the distinct time points on which state-changes occurred in $\partobs$, and the intervals $[t_{i-1},t_i)$ are time intervals during which the process remained in the same state it had at time $t_{i-1}$. The boundaries $t_0$ and $t_M$ constitute special cases that are included for use in the proof of the next lemma. Notably, there is never a state-change at $t_0=0$, and there might, but need not be, a state-change at time $t_M=t_\mathrm{max}$.

The above guarantees that the properties of the discretised realisations $\partobsdisc$ converge to the properties of the original $\partobs$. 
Unfortunately, a straightforward statement of the results that we need is (as before, apparently) so elementary that we are unable to find a satisfactory reference. We therefore provide an explicit proof below.
\begin{lemma}\label{lemma:path_props}
For any c\`adl\`ag $\partobs$ and all $x\in\states$:
\begin{enumerate}[label=(\roman*)]
\item $n_{xy}=\lim_{m\to+\infty} \discretetranscount{xy}$ for all $y\in\states\setminus\{x\}$;\label{lemma:path_props:partition_fine_enough}
\item $d_x = \lim_{m\to+\infty}\discretestep\discretecounts{x}$.\label{lemma:path_props:durations_converge}
\end{enumerate}
\end{lemma}
\begin{proof}
Let $t_i\in[0,t_\mathrm{max}]$, $i=0,\ldots,M$ be the finite set of time points 
whose existence is guaranteed by Lemma~\ref{lemma:cadlag_has_finite_switches}. Let $\Delta\coloneqq \min_{i\in\{1,\ldots,M\}}t_i-t_{i-1}$ be the 
minimum distance between these time points.

We start by proving Property~\ref{lemma:path_props:partition_fine_enough}. Fix $x,y\in\states$ such that $x\neq y$. Because $\discretestep=\nicefrac{t_\mathrm{max}}{m}$, there is some $N\in\nats$ such that, for all $m>N$, $\discretestep<\Delta$. Fix any such $m>N$. 

Recall that $\discretetranscount{xy}$ is the number of transitions from $x$ to $y$ in $\partobsdisc$. Thus,
\begin{equation*}
\discretetranscount{xy} = \sum_{j=1}^{m}\ind{x}\bigl(\partobsdisc(j-1)\bigr)\ind{y}\bigl(\partobsdisc(j)\bigr)\,.
\end{equation*}
Let $\mathcal{I}_{xy}$ consist of the indices $i$ of the time points $t_i$ at which the actual switches from state $x$ to $y$ occurred in $\partobs$; so, let
\begin{equation*}
\mathcal{I}_{xy} \coloneqq \Bigl\{ i\in\{1,\ldots,M\}\,:\,\partobs(t_{i-1})=x, \partobs(t_{i})=y  \Bigr\}.
\end{equation*}
Then, clearly, $n_{xy}=\lvert\mathcal{I}_{xy}\rvert$ is the true number of transitions from $x$ to $y$.

Choose any $i\in\mathcal{I}_{xy}$. Clearly, since $\discretestep<\Delta$, there is a unique $j_i\in\{1,\ldots,m\}$ such that 
\begin{equation*}
t_{i-1}\leq t_i-\Delta < (j_i-1)\discretestep < t_i \leq j_i\discretestep < t_i+\Delta.
\end{equation*}
Then $\partobsdisc(j_i-1)=x$ and $\partobsdisc(j_i)=y$ because $\partobs(t_{i-1})=x$ and $\partobs(t_i)=y$. Therefore $\ind{x}\bigl(\partobsdisc(j_i-1)\bigr)\ind{y}\bigl(\partobsdisc(j_i)\bigr) =1$. Because this holds for all $i\in\mathcal{I}_{xy}$, and because each $i\in\mathcal{I}_{xy}$ has a unique $j_i$, this implies that $\discretetranscount{xy}\geq \lvert\mathcal{I}_{xy}\rvert=n_{xy}$.

Conversely, it trivially holds that $n_{xy} \geq \discretetranscount{xy}$ because the discretisation cannot introduce more state switches. Therefore we must have that $\discretetranscount{xy}=n_{xy}$. Because this holds for all $m>N$, it holds that $\lim_{m\to+\infty}\discretetranscount{xy}=n_{xy}$, which concludes the proof for Property~\ref{lemma:path_props:partition_fine_enough}.

We next prove property Property~\ref{lemma:path_props:durations_converge}; choose any $x\in\states$, and recall that
\begin{equation}\label{eq:duration_integral}
d_x \coloneqq \int_0^{t_\mathrm{max}} \mathbb{I}_x\bigl(\partobs(t)\bigr)\,\mathrm{d}t\,,
\end{equation}
where $\mathbb{I}_x$ is the indicator of $\{x\}$. By Lemma~\ref{lemma:cadlag_has_finite_switches}, $\partobs$ has only finitely many discontinuities on the interval $[0,t_\mathrm{max}]$. Therefore, the composite function $\mathbb{I}_x\bigl(\partobs(t)\bigr)$ also has only finitely many discontinuities on this interval. It follows that the integral in~\eqref{eq:duration_integral} can be interpreted in the Riemann sense. 

Now consider any $m\in\nats$. Then we have
\begin{equation*}
\discretestep \discretecounts{x} = \sum_{i=0}^{m-1} \mathbb{I}_x\bigl(\partobsdisc(i)\bigr)\discretestep = \sum_{i=0}^{m-1} \mathbb{I}_x\left(\partobs\left(\frac{i}{m}t_\mathrm{max}\right)\right)\frac{t_\mathrm{max}}{m}\,,
\end{equation*}
which we see is a Riemann sum whose limit defines the integral in~\eqref{eq:duration_integral}; we immediately conclude that $d_x=\lim_{m\to_\infty}\discretestep\discretecounts{x}$, as claimed. \qed
\end{proof}

\emph{Proof of Theorem~\ref{thm:limit_of_idm}}. 
We need to show that $\lim_{m\to+\infty}\mathcal{Q}_s^{(m)}$ exists in the Painlev\'e-Kuratowski sense~\cite{rockafellar1998variational}, and that it is equal to $\mathcal{Q}_s$. This requires us to consider the inner limit of $\smash{\{\mathcal{Q}_s^{(m)}\}_{n\in\nats}}$---the set of limit points of sequences $\smash{\{Q_m\}_{m\in\nats}}$, with $\smash{Q_m\in\mathcal{Q}_s^{(m)}}$ for all $m\in\nats$---and the outer limit---the set of all accumulation points of such sequences---and to show that they are equal to each other and to $\mathcal{Q}_s$. We start by considering the inner limit.

Fix any $Q\in\mathcal{Q}_s$. It then follows from~\eqref{eq:proposed_estimator} that $Q\in\mathfrak{Q}$ and that there is some $A\in\transmatset$ such that
\begin{equation}\label{eq:prooftheorem5:1}
Q(x,y)=\frac{sA(x,y)+n_{xy}}{d_x}
\text{ for all $x,y\in\states$ such that $x\neq y$.}
\end{equation}
Since $A\in\transmatset$, we know that there must be sequence $\{A_m\}_{m\in\nats}\in\mathrm{int}(\mathfrak{T})$ such that $\lim_{m\to+\infty}A_m=A$. Consider any such sequence.

For all $m\in\nats$, we now let $T_m$ be the element of $\smash{\mathcal{T}^{(m)}_s}$ that corresponds to $A_m$, and we let $\smash{Q_m\coloneqq(T_m-I)\nicefrac{1}{\delta^{(m)}}}$ be the corresponding element of $\smash{\mathcal{Q}^{(m)}_s}$. Consider now any $x,y\in\states$ such that $x\neq y$. For all $m\in\nats$, we then find that
\begin{equation}\label{eq:prooftheorem5:2}
Q_m(x,y) = \frac{s A_m(x,y) + \discretetranscount{xy}}{\discretestep s+\discretestep n_x^{(m)}} - \frac{I(x,y)}{\discretestep}
= \frac{s A_m(x,y) + \discretetranscount{xy}}{\discretestep s+\discretestep n_x^{(m)}},
\end{equation}
because $x\neq y$ implies $I(x,y)=0$. Furthermore, we also know that
\begin{equation*}
\begin{array}{lll}
\lim_{m\to+\infty} A_m(x,y) = A(x,y), & \quad\quad & \lim_{m\to+\infty} \discretetranscount{xy} = n_{xy}, \\
\lim_{m\to+\infty} \discretestep s = 0, &  & \lim_{m\to+\infty} \discretestep \discretecounts{x} = d_x,
\end{array}
\end{equation*}
making use of Lemma~\ref{lemma:path_props} for the equalities that involve $n_{xy}$ and $d_x$. Therefore, the numerator and denominator converge separately, and we find that
\begin{equation*}
\lim_{m\to+\infty} Q_m(x,y) = \frac{sA(x,y) + n_{xy}}{d_x}=Q(x,y).
\end{equation*}
It remains to look at the diagonal elements. Fix any $x\in\states$. Then
\begin{align*}
\lim_{m\to+\infty}Q_m(x,x) &= \lim_{m\to+\infty}-\sum_{y\in\states\setminus\{x\}}Q_m(x,y)\\
&=-\sum_{y\in\states\setminus\{x\}}\lim_{m\to+\infty}Q_m(x,y)
=-\sum_{y\in\states\setminus\{x\}}Q(x,y)=Q(x,x),
\end{align*}
where the first equality follows from the fact that each $Q_m$ is a transition rate matrix, the third equality follows from our earlier result that $Q_m(x,y)$ converges to $Q(x,y)$, and the last equality follows because $Q$ is a transition rate matrix.
We conclude that $\lim_{m\to+\infty}Q_m=Q$. Since $\smash{Q_m\in\mathcal{Q}^{(m)}_s}$ for all $m\in\nats$, this implies that $\smash{Q\in \liminf_{m\to+\infty} \mathcal{Q}^{(m)}_s}$, where $\smash{\liminf_{m\to+\infty} \mathcal{Q}^{(m)}_s}$ is the inner limit of $\smash{\{\mathcal{Q}_s^{(m)}\}_{n\in\nats}}$. Because this holds for all $Q\in\mathcal{Q}_s$, we conclude that
\begin{equation}\label{eq:subset_inner_limit}
\mathcal{Q}_s\subseteq \liminf_{m\to+\infty}\mathcal{Q}^{(m)}_s.
\end{equation}

Next, consider any element $Q\in \limsup_{m\to+\infty}\mathcal{Q}^{(m)}_s$ of the outer limit. By definition, there is then a sequence $\smash{\{Q_m\}_{m\in\nats}}$, with $\smash{Q_m\in\mathcal{Q}_s^{(m)}}$ for all $m\in\nats$, and a subsequence $\smash{\{Q_{m_\ell}\}_{\ell\in\nats}}$, such that $\lim_{\ell\to+\infty}Q_{m_\ell} = Q$. 

For every $m\in\nats$, since $Q_{m}\in\mathcal{Q}^{(m)}_s$, we know that there is some $A_{m}\in\mathrm{int}(\transmatset)$ that satisfies~\eqref{eq:prooftheorem5:2} for all $x,y\in\states$ such that $x\neq y$. Furthermore, because $\{A_m\}_{m\in\nats}\subseteq\mathrm{int}(\transmatset)\subset\transmatset$ and $\transmatset$ is compact, it follows from the Bolzano-Weierstrass theorem that the sequence $\{A_{m_\ell}\}_{\ell\in\nats}$ has a convergent subsequence whose limit belongs to $\transmatset$. Hence, without loss of generality, we can assume that $\lim_{\ell\to+\infty}A_{m_\ell}=A$, with $A\in\transmatset$.

Using completely analogous argumentation as that in the first part of this proof, it now follows that $Q=\lim_{\ell\to+\infty}Q_{m_\ell}$ satisfies~\eqref{eq:prooftheorem5:1}. It follows that the off-diagonal elements of $Q$ are real-valued and, similar to what we found above, that the rows of $Q$ sum to zero; hence, the diagonal elements are real-valued as well. Therefore $Q$ is a transition rate matrix and, since it satisfies~\eqref{eq:prooftheorem5:1}, this implies that $Q\in\mathcal{Q}_s$. Since $\smash{Q\in\limsup_{m\to+\infty}\mathcal{Q}^{(m)}_s}$ was arbitrary, we conclude that $\smash{\limsup_{m\to+\infty}\mathcal{Q}^{(m)}_s \subseteq\mathcal{Q}_s}$.
Since the inner limit is trivially included in the outer one, the result now follows from~\eqref{eq:subset_inner_limit}.\qed

\end{document}